\documentclass[twoside]{article}

\usepackage[accepted]{aistats2024}
\usepackage{amsmath}
\usepackage{amsthm}
\usepackage{amssymb}
\usepackage{braket}

\newtheorem{assumption}{Assumption}

\newtheorem{corollary}{Corollary}
\newtheorem{theorem}{Theorem}

\newtheorem{lemma}{Lemma}

\usepackage{array}
\newcolumntype{M}[1]{>{\centering\arraybackslash}m{#1}}
\newcolumntype{N}{@{}m{0pt}@{}}

\usepackage{footnote}
\makesavenoteenv{tabular}
\makesavenoteenv{table}

\usepackage{algorithm}
\usepackage[noend]{algpseudocode}

\usepackage[round]{natbib}

\begin{document}

\runningtitle{Improved Sample Complexity of NPG with General Parameterization}


\twocolumn[

\aistatstitle{Improved Sample Complexity Analysis of Natural Policy Gradient Algorithm with General Parameterization for Infinite Horizon Discounted Reward Markov Decision Processes}

\aistatsauthor{ Washim Uddin Mondal \And  Vaneet Aggarwal }

\aistatsaddress{ Purdue University, USA} ]

\begin{abstract}
We consider the problem of designing sample efficient learning algorithms for infinite horizon discounted reward Markov Decision Process. Specifically, we propose the Accelerated Natural Policy Gradient (ANPG) algorithm that utilizes an accelerated stochastic gradient descent process to obtain the natural policy gradient. ANPG achieves $\mathcal{O}({\epsilon^{-2}})$ sample complexity and $\mathcal{O}(\epsilon^{-1})$ iteration complexity with general parameterization where $\epsilon$ defines the optimality error. This improves the state-of-the-art sample complexity by a $\log(\frac{1}{\epsilon})$  factor. ANPG is a first-order algorithm and unlike some existing literature, does not require the unverifiable assumption that the variance of importance sampling (IS) weights is upper bounded. In the class of Hessian-free and IS-free algorithms, ANPG beats the best-known sample complexity by a factor of $\mathcal{O}(\epsilon^{-\frac{1}{2}})$ and simultaneously matches their state-of-the-art iteration complexity.
\end{abstract}
\section{Introduction}

Reinforcement Learning (RL) is a sequential decision-making framework that finds its applications in a wide range of areas, ranging from epidemic control to transportation to online marketing \citep{ling2023cooperating, singh2023maximize, al2019deeppool}. The objective of an RL agent is to obtain a policy that maximizes the discounted sum of expected rewards that are generated due to its interaction with the environment. One way to solve this problem is called the policy gradient (PG) approach which performs the optimization directly in the policy space. Value function-based approaches, on the contrary, optimize the $Q$-functions which are then used to filter out the policies. 

Paired with the general function approximation, PG-based methods turn out to be an effective tool in handling large state spaces. Specifically, deep neural network (DNN)-based PG has shown tremendous success empirically \citep{schulman2015trust,schulman2017proximal}. The idea of function approximation is to parameterize the policies by a $\mathrm{d}$-dimensional parameter, $\theta$ so that the optimization can be performed over $\mathbb{R}^{\mathrm{d}}$. In this case, however, the objective function, $J(\theta)$, appears to be non-convex. As a result, many articles primarily analyze the first-order stationary (FOS) convergence properties of $J(\theta)$. For example, \citep{yuan2022general} derives $\tilde{O}(\epsilon^{-4})$ sample complexity for Vanilla-PG algorithm to achieve an  $\epsilon$-FOS error.

\begin{table*}[ht]
    \centering
    \begin{tabular}{|c|c|c|c|c|}
    \hline
    Algorithm & Sample Complexity & Iteration Complexity & Hessian-free & IS-free \\ 
    \hline 
    & & & &\\[-0.35cm]
    Vanilla-PG \citep{yuan2022general} & $\Tilde{\mathcal{O}}(\epsilon^{-3})$ & $\mathcal{O}(\epsilon^{-3})$ &Yes & Yes \\
    \hline
    & & & &\\[-0.35cm]
    STORM-PG-F \citep{ding2022global} & $\tilde{\mathcal{O}}((1+W)^{\frac{3}{2}}\epsilon^{-3})$ & $\mathcal{O}(\epsilon^{-3})$ & Yes & No \\
    \hline
    & & & &\\[-0.35cm]
    SCRN \citep{masiha2022stochastic} & $\tilde{\mathcal{O}}(\epsilon^{-2.5})$ & $\mathcal{O}(\epsilon^{-0.5})$ & No & Yes \\
    \hline
    & & & &\\[-0.35cm]
    VR-SCRN \citep{masiha2022stochastic} & $\mathcal{O}(\epsilon^{-2}\log\left(\frac{1}{\epsilon}\right))$ & $\mathcal{O}(\epsilon^{-0.5})$ & No & No \\[0.05cm]
    \hline
    & & & &\\[-0.35cm]
    NPG \citep{liu2020improved} & $\mathcal{O}(\epsilon^{-3})$ & $\mathcal{O}(\epsilon^{-1})$ & Yes & Yes \\
    \hline
    & & & &\\[-0.35cm]
    SRVR-NPG \citep{liu2020improved} & $\mathcal{O}(W\epsilon^{-2.5}+\epsilon^{-3})$ & $\mathcal{O}(\epsilon^{-1})$ & Yes & No \\
    \hline
    & & & &\\[-0.35cm]
    SRVR-PG \citep{liu2020improved} & $\mathcal{O}(W\epsilon^{-3})$ & $\mathcal{O}(\epsilon^{-2})$ & Yes & No \\
    \hline
    & & & &\\[-0.35cm]
    N-PG-IGT \citep{fatkhullin2023stochastic} & $\tilde{\mathcal{O}}(\epsilon^{-2.5})$ & $\mathcal{O}(\epsilon^{-2.5})$ & Yes & Yes\\ 
    \hline
    & & & &\\[-0.35cm]
    HARPG \citep{fatkhullin2023stochastic} & $\mathcal{O}(\epsilon^{-2}\log\left(\frac{1}{\epsilon}\right))$ & $\mathcal{O}(\epsilon^{-2})$ & No & Yes\\ [0.05cm]
    \hline
    & & & &\\[-0.35cm]
    \textbf{ANPG (This work)}  & $\mathcal{O}(\epsilon^{-2})$ & $\mathcal{O}(\epsilon^{-1})$ & Yes & Yes \\[0.05cm]
    \hline
    & & & &\\[-0.35cm]
    Lower Bound \citep{azar2017minimax} & $\mathcal{O}(\epsilon^{-2})$ & $-$ & $-$ & $-$\\
    \hline
    \end{tabular}
    \caption{Summary of recent sample complexity results for global convergence in discounted reward MDPs with general parameterized policies. The term, ``Hessian-free" states that the underlying algorithm only utilizes first-order information. The ``IS-free" column states whether a bound $(W)$ on the variance of the importance sample (IS) weights is assumed. It is worthwhile to mention here that although the HARPG algorithm is Hessian-aided, its per-iteration computational and memory requirements are similar to that of Hessian-free methods \citep{fatkhullin2023stochastic}. We have explicitly written the logarithmic terms in the sample complexities of HARPG and VR-SCRN to highlight the superiority of our result.}
    \label{tab:related_works}
\end{table*}

In this paper, however, we are interested in the global convergence property of $J(\theta)$. It can be shown that under the assumption of Fisher Non-Degeneracy (FND), $J(\theta)$ satisfies the gradient domination property which implies that  there does not exist any FOS sub-optimal points of $J(\theta)$. As a result, one can achieve any $\epsilon$ global optimum (GO) gap, despite the objective function being non-convex. For example, \citep{liu2020improved} derives $\mathcal{O}(\epsilon^{-4})$ and $\mathcal{O}(\epsilon^{-3})$ sample complexities of the PG and the Natural PG (NPG) algorithms respectively corresponding to an $\epsilon$  GO gap. In recent years, importance sampling (IS) has garnered popularity as an important sample complexity-reducing method. \citep{liu2020improved} utilizes IS to propose  variance-reduction (VR) versions of PG and NPG algorithms, each achieving a sample complexity of $\tilde{\mathcal{O}}(\epsilon^{-3})$. The caveat of IS-based methods is that one needs to impose a bound on the variance of the IS weights. This is a strong assumption and might not be verifiable for most practical cases. Alternatives to the IS-based method include the momentum-based N-PG-IGT algorithm of \citep{fatkhullin2023stochastic} and the Hessian-based SCRN algorithm of \citep{masiha2022stochastic}, both of which achieve $\tilde{\mathcal{O}}(\epsilon^{-2.5})$ sample complexity. Hessian-based (second-order) approaches typically have more memory 
and computational requirements in comparison to the first-order algorithms and thus, are less preferred in practice.
The state-of-the-art (SOTA)
$\tilde{\mathcal{O}}(\epsilon^{-2})$ sample complexity is achieved by two recently proposed algorithms: the VR-SCRN by
\citep{masiha2022stochastic}
and the HARPG proposed by \citep{fatkhullin2023stochastic}. The first one uses a combination of Hessian and IS-based methods whereas the other one solely applies a Hessian-aided technique. This leads to the following question.

\begin{center}
\fbox{
\parbox{0.9\linewidth}{Does there exist an IS-free and Hessian-free algorithm that either achieves or improves the SOTA $\tilde{\mathcal{O}}(\epsilon^{-2})$ sample complexity?}
}    
\end{center}

\subsection{Our Contributions and Challenges}

In this paper, we provide an affirmative answer to  the above question. In particular,

$\bullet$ We propose an acceleration-based NPG (ANPG) algorithm that uses accelerated stochastic gradient descent (ASGD) to determine an estimate of the natural policy gradient.

$\bullet$ We prove that ANPG achieves a sample complexity of $\mathcal{O}(\epsilon^{-2})$ and an iteration complexity of $\mathcal{O}(\epsilon^{-1})$.

$\bullet$ ANPG is IS-free and does not need any second-order (Hessian-related) computation.

Note that the best known sample complexity is $\tilde{\mathcal{O}}(\epsilon^{-2})$ $=\mathcal{O}\left(\epsilon^{-2}\log\left(\frac{1}{\epsilon}\right)\right)$. Clearly, our proposed ANPG algorithm improves the SOTA  sample complexity by a factor of $\log\left(\frac{1}{\epsilon}\right)$. 
Additionally, in the class of Hessian-free and IS-free algorithms, it improves the SOTA sample complexity
by a factor of $\epsilon^{-0.5}$  (see Table \ref{tab:related_works}).  Note that our sample complexity matches the lower bound and hence, cannot be improved further. Moreover, the iteration complexity of our algorithm matches the SOTA iteration complexity of Hessian-free and IS-free algorithms and beats that of the HARPG algorithm by a factor of $\mathcal{O}(\epsilon^{-1})$.

The NPG algorithm updates the policy parameter, $\theta$ for $K$ iterations and at each iteration, $k\in\{0,  \cdots, K-1\}$, at estimate, $\omega_k$, of the true natural gradient, $\omega_k^*$ is calculated via an $H$-step SGD method. The iteration complexity of the NPG, as shown in \citep{liu2020improved} is $\mathcal{O}(K)=\mathcal{O}(\epsilon^{-1})$. This is difficult to improve since even with exact natural gradients and softmax parameterization, we obtain the same result \citep[Theorem 16]{agarwal2021theory}. On the other hand, the estimation error bound $\mathbb{E}\Vert \omega_k - \omega_k^* \Vert = \mathcal{O}(1/\sqrt{H})$ is standard and hard to change. This leads to a choice of $H=\mathcal{O}(\epsilon^{-2})$, resulting in a sample complexity of $\mathcal{O}(\epsilon^{-3})$. Therefore, the sample complexity of the NPG, in its current form, seems to be hard to break. We introduce the following modifications to surpass this apparent barrier. Firstly, the intermediate direction finding SGD routine is substituted by the ASGD procedure. Next, by improving the global convergence analysis, we show that the first-order estimation error can be more accurately written as $\mathbf{E}\Vert (\mathbf{E}[\omega_k|\theta_k]-\omega_k^*)\Vert$ where $\theta_k$ is the policy parameter at the $k$th iteration. Finally, we prove that the stated  term can be interpreted as the approximation error of a noiseless (deterministic) accelerated gradient descent process and hence can be upper bounded as $\exp(-\kappa H)$ for some $H$-independent parameter $\kappa$. 
This allows us to take 
$H=\mathcal{O}(\epsilon^{-1})$, and resulting in $\mathcal{O}(\epsilon^{-2})$ sample complexity.

\subsection{Related Works}

$\bullet$ \textbf{Convergence with Exact Gradients:} A series of papers in the literature explore the global convergence rate (or equivalently, iteration complexity) of PG-type algorithms assuming access to the exact gradients. For example, \citep{agarwal2021theory} shows $\mathcal{O}(\epsilon^{-2})$ iteration complexity of the exact PG with softmax parameterization and log barrier regularization. \citep{mei2020global} obtains an improved $\mathcal{O}(\epsilon^{-1})$ result without regularization and $\mathcal{O}(\log\left(\frac{1}{\epsilon}\right))$ result with regularization. \citep{shani2020adaptive} proves   $\mathcal{O}(\epsilon^{-1})$ iteration complexity for the exact regularized NPG. 
\citep{agarwal2021theory} proves an identical result for the unregularized NPG.
Finally, $\mathcal{O}(\log\left(\frac{1}{\epsilon}\right))$ iteration complexity is recently shown both for regularized \citep{cen2022fast,lan2023policy,zhan2023policy} and unregularized \citep{bhandari2021linear, xiao2022convergence} NPG. %Recently, \citep{yuan2022linear} has extended the result of \citep{xiao2022convergence} for tabular setup to stochastic gradients with log-linear policies and shown logarithmic iteration and quadratic sample complexities.

$\bullet$ \textbf{FOS Sample Complexity:} Many variance reduction methods exist in the literature that reaches the $\epsilon$-FOS point faster than the Vanilla-PG algorithm \citep{yuan2022general}. The majority of these works incorporate some form of the importance sampling (IS) technique. Works that fall into this category include \citep{papini2018stochastic, xu2019sample, xu2020improved, gargiani2022page, pham2020hybrid, yuan2020stochastic, huang2020momentum}, etc. As an alternative to the IS-based method, some papers  \citep{salehkaleybar2022adaptive, shen2019hessian} resort to second-order (Hessian-related) information. The techniques stated above are sometimes used in conjunction with momentum-based updates \citep{salehkaleybar2022adaptive, huang2020momentum}. 

$\bullet$ \textbf{GO Sample Complexity:} Most of the works that target global optimal (GO) convergence using stochastic gradients have been discussed earlier. Alongside, a plethora of works has emerged that investigate the GO convergence via actor-critic methods with linear function approximation \citep{wu2020finite, chen2022finite, chen2022sample, khodadadian2022finite, qiu2021finite}. These works assume the underlying MDP to be ergodic. Fortunately, such an assumption is not a requirement for our analysis. As mentioned before, we
assume Fisher Non-Degenerate (FND) policies and the compatible function approximation framework (defined in section \ref{sec:samle_compexity}) to arrive at our result. A similar  analysis of the PG algorithm with general parameterization has been recently explored for average reward MDPs \citep{bai2023regret}.

\section{Problem Formulation}

We consider a Markov Decision Process (MDP) characterized by a tuple $\mathcal{M}=(\mathcal{S}, \mathcal{A}, r, P, \gamma, \rho)$ where $\mathcal{S}, \mathcal{A}$ are the state space and the action space respectively, $r:\mathcal{S}\times \mathcal{A}\rightarrow [0, 1]$ is the reward function, $P:\mathcal{S}\times \mathcal{A}\rightarrow \Delta^{|\mathcal{S}|}$ is the state transition kernel (where $\Delta^{|\mathcal{S}|}$ indicates a probability simplex of size $|\mathcal{S}|$), $\gamma\in (0, 1)$ denotes the reward discount factor, and $\rho\in\Delta^{|\mathcal{S}|}$ is the initial state distribution. A policy $\pi:\mathcal{S}\rightarrow \Delta^{|\mathcal{A}|}$ characterizes the distribution of action to be executed given the current state. For a state-action pair $(s, a)$, the $Q$-function of a policy $\pi$ is defined as follows.
\begin{align*}
    Q^{\pi}(s, a) \triangleq \mathbf{E}\left[\sum_{t=0}^{\infty} \gamma^t r(s_t, a_t)\bigg| s_0=s, a_0=a\right]
\end{align*}
where the expectation is computed over all $\pi$-induced trajectories $\{(s_t, a_t)\}_{t=0}^{\infty}$ where $s_{t}\sim P(s_{t-1}, a_{t-1})$, and $a_t\sim \pi(s_t)$, $\forall t\in\{1, 2, \cdots\}$. Similarly, the $V$-function of the policy $\pi$ is defined as,
\begin{align*}
\begin{split}
        V^{\pi}(s) &\triangleq \mathbf{E}\left[\sum_{t=0}^{\infty} \gamma^t r(s_t, a_t)\bigg| s_0=s\right] \\
        &= \sum_{a\in\mathcal{A}} \pi(a|s)Q^{\pi}(s, a)
\end{split}
\end{align*}

Finally, the advantage function is defined as,
\begin{align*}
    A^{\pi}(s, a) \triangleq Q^{\pi}(s, a) - V^{\pi}(s), ~\forall (s, a)\in\mathcal{S}\times\mathcal{A}
\end{align*}

Our objective is to maximize the function defined below over the class of all policies.
\begin{align*}
    J^{\pi}_{\rho} \triangleq \mathbf{E}_{s\sim \rho}\left[V^{\pi}(s)\right] = \dfrac{1}{1-\gamma}\sum_{s, a}d^{\pi}_{\rho}(s, a)\pi(a|s)r(s, a)
\end{align*}
where the state occupancy $d^{\pi}_{\rho}\in\Delta^{|\mathcal{S}|}$ is defined as,
\begin{align*}
\begin{split}
    d^{\pi}_{\rho}(s) &\triangleq (1-\gamma)\sum_{t=0}^\infty \gamma^t \mathrm{Pr}(s_t=s|s_0\sim\rho, \pi),~\forall s\in\mathcal{S}
\end{split}
\end{align*}
We define the state-action occupancy  $\nu^{\pi}_\rho\in\Delta^{|\mathcal{S}\times \mathcal{A}|}$ as $\nu_\rho^{\pi}(s, a) \triangleq d^{\pi}_\rho(s)\pi(a|s)$, $\forall (s, a)\in\mathcal{S}\times \mathcal{A}$. 

In many application scenarios, the state space is quite large which calls for the parameterization of the policy functions by a deep neural network with $\mathrm{d}$-dimensional parameters. Let, $\pi_{\theta}$ indicate the policy corresponding to the parameter $\theta\in \mathbb{R}^{\mathrm{d}}$. In this case, our maximization problem can be restated as shown below.
\begin{align}
    \max_{\theta\in\mathbb{R}^{\mathrm{d}}}~J^{\pi_{\theta}}_{\rho} 
    \label{eq:max_J}
\end{align}
For notational convenience, we denote $J_\rho^{\pi_\theta}$ as $J_\rho(\theta)$ in the rest of the paper.

\section{Algorithm}

One way to solve the maximization $(\ref{eq:max_J})$ is via updating the policy parameters by applying the gradient ascent: $\theta_{k+1} = \theta_k + \eta \nabla_{\theta}J_{\rho}(\theta_k)$, $k\in\{0, 1, \cdots\}$ starting with an initial parameter, $\theta_0$. Here, $\eta>0$ denotes the learning rate, and the policy gradients (PG) are given as follows \citep{sutton1999policy}.
\begin{equation}
%\begin{split}
        \nabla_{\theta} J_{\rho}(\theta)
        =\dfrac{1}{1-\gamma}H_\rho(\theta), \text{ where }  
%\end{split}
\label{eq:pg_expression}
\end{equation}
%where $H_\rho(\theta)$ is given as,
\begin{equation}
  H_{\rho}(\theta) \triangleq \mathbf{E}_{(s, a)\sim \nu_{\rho}^{\pi_\theta}}\big[A^{\pi_\theta}(s, a)\nabla_{\theta}\log \pi_{\theta}(a|s)\big]\label{hrho}
\end{equation}

In contrast, this paper uses the natural policy gradient (NPG) to update the policy parameters. In particular, $\forall k\in\{1, 2, \cdots\}$, we have,
\begin{align}
    \theta_{k+1} = \theta_k + \eta F_{\rho}(\theta_k)^{\dagger}\nabla_{\theta} J_{\rho}(\theta_k)
\end{align}
where $\dagger$ is the Moore-Penrose pseudoinverse operator, and $F_{\rho}$ is called the Fisher information function which is defined as follows $\forall \theta\in \mathbb{R}^{\mathrm{d}}$.
\begin{align}
\begin{split}
    F_{\rho}(\theta) \triangleq \mathbf{E}_{(s, a)\sim \nu^{\pi_\theta}_{\rho}}\big[
    \nabla_{\theta}\log \pi_{\theta}(a|s)\otimes\nabla_{\theta}\log \pi_{\theta}(a|s)\big]
\end{split} 
\label{eq:def_F_rho_theta}
\end{align}
where $\otimes$ denotes the outer product. In particular, for any $\mathbf{a}\in\mathbb{R}^{\mathrm{d}}$, $\mathbf{a}\otimes \mathbf{a} = \mathbf{aa}^{\mathrm{T}}$.

Note that, in comparison to a PG-based iteration, the advantage of an NPG-based iteration is that the latter modulates the learning parameter, $\eta$, according to the rate of change of the policy function itself. This results in a lower sample complexity as shown in the previous study \citep{liu2020improved}. As mentioned before, our goal in this paper is to show that the sample complexity of the NPG-based update can be reduced further. 

Let $\omega^*_{\theta}\triangleq F_{\rho}(\theta)^{\dagger}\nabla_{\theta} J_{\rho}(\theta)$. Notice that we have removed the dependence of $\omega^*_{\theta}$ on $\rho$ for notational convenience. Invoking $\eqref{eq:pg_expression}$, the term, $\omega_{\theta}^*$ can be written as a solution of a quadratic optimization \citep{peters2008natural}. Specifically, $\omega_{\theta}^* = {\arg\min}_{\omega} L_{\nu^{\pi_\theta}_\rho}(\omega, \theta)$ where $L_{\nu^{\pi_\theta}_\rho}(\omega, \theta)$ is the compatible function approximation error and it is mathematically defined as,
\begin{align}
\label{eq:def_app_error}
\begin{split}
    ~L_{\nu^{\pi_\theta}_\rho}(\omega, \theta)\triangleq \dfrac{1}{2}\mathbf{E}_{(s, a)\sim \nu^{\pi_\theta}_{\rho}}&\bigg[\dfrac{1}{1-\gamma}A^{\pi_{\theta}}(s, a)  \\
    &\hspace{0.0cm} -  \omega^{\mathrm{T}}\nabla_{\theta}\log \pi_{\theta}(a|s)\bigg]^2
\end{split}
\end{align}

Using the above notations, NPG updates can be written as $\theta_{k+1}=\theta_k+\eta \omega_k^*$, $k\in\{1, 2, \cdots\}$ where $\omega_k^*=\omega^*_{\theta_k}$. However, in most practical scenarios, the learner does not have knowledge of the state transition probabilities which makes it difficult to directly calculate $\omega_{k}^*$. Below we clarify how one can get its sample-based estimates. 
Note that the gradient of $L_{\nu^{\pi_\theta}_\rho}(\cdot, \theta)$ can be obtained as shown below for any arbitrary $\theta$. \hspace{3cm}
\begin{align}
    \nabla_{\omega}L_{\nu^{\pi_\theta}_\rho}(\omega, \theta) = F_{\rho}(\theta)\omega - \dfrac{1}{1-\gamma}H_{\rho}(\theta)
\end{align}
%where $F_\rho(\theta)$ and $H_\rho(\theta)$ are defined in $(\ref{eq:def_F_rho_theta})$ and \eqref{hrho}, respectively. 

Although, in general, it is difficult to obtain $F_{\rho}(\theta)$ and $H_\rho(\theta)$ exactly, their (unbiased) estimates can be easily obtained via Algorithm \ref{algo_sampling}. This is similar to Algorithm 3 stated in \citep{agarwal2021theory}. The Algorithm first runs the MDP for $T$ time instances following the policy $\pi_{\theta}$ starting from $s_0\sim \rho$ where $T$ is a geometric random variable with success probability $1-\gamma$. The tuple $(s_T, a_T)$ is taken as the state-action pair $(\hat{s}, \hat{a})$ sampled from the distribution $\nu^{\pi_\theta}_{\rho}$. With probability $\frac{1}{2}$, the algorithm then generates another $\pi_{\theta}$-induced trajectory of length $T\sim \mathrm{Geo}(1-\gamma)$ starting from the pair $(\hat{s}, \hat{a})$, and assigns twice of the total reward of this trajectory as the estimate $\hat{Q}^{\pi_\theta}(\hat{s}, \hat{a})$. In its complementary event, the trajectory initiates with $\hat{s}$, and the twice of its total reward is assigned as $\hat{V}^{\pi_\theta}(\hat{s})$. Its associated advantage estimate is obtained as $\hat{A}^{\pi_{\theta}}(\hat{s}, \hat{a})=\hat{Q}^{\pi_{\theta}}(\hat{s}, \hat{a})-\hat{V}^{\pi_{\theta}}(\hat{s})$. Finally, the estimate of the gradient is computed as,
\begin{align}
    \hat{\nabla}_\omega L_{\nu^{\pi_\theta}_\rho}(\omega, \theta) =  \hat{F}_{\rho}(\theta)\omega - \dfrac{1}{1-\gamma}\hat{H}_{\rho}(\theta)
\end{align}
where $\hat{F}_{\rho}(\theta)$ and $\hat{H}_{\rho}(\theta)$ are given as follows.
\begin{align}
    \begin{split}
        &\hat{F}_{\rho}(\theta) = \nabla_{\theta}\log\pi_{\theta}(\hat{a}|\hat{s})\otimes\nabla_{\theta}\log\pi_{\theta}(\hat{a}|\hat{s})\\
        &\hat{H}_{\rho}(\theta) = \hat{A}^{\pi_\theta}(\hat{s}, \hat{a})\nabla_{\theta}\log \pi_\theta(\hat{a}|\hat{s})
    \end{split}
\end{align}
The following lemma ensures that the estimate of the gradient generated by Algorithm \ref{algo_sampling} is indeed unbiased.  

\begin{algorithm}[t]
    \begin{algorithmic}[1]
    \caption{Unbiased Sampling}
        \State \textbf{Input:} Parameters $(\theta,\omega)$, Initial Distribution $\rho$
        \vspace{0.2cm}
        \State $T\sim \mathrm{Geo}(1-\gamma)$
        \State Sample $s_0\sim \rho$, $a_0\sim \pi_{\theta}(s_0)$
        \For{$j\in\{0,\cdots, T-1\}$}
            \State Execute $a_{j}$ and observe $s_{j+1}\sim P(s_j, a_j)$
            \State Sample $a_{j+1}\sim \pi_{\theta}(s_{j+1})$
        \EndFor
        \State $(\hat{s}, \hat{a})\gets (s_{T}, a_{T})$
        \vspace{0.2cm}
        \State $(s_0, a_0)\gets (s_{T}, a_{T})$
        \State $T\sim \mathrm{Geo}(1-\gamma)$, $X\sim \mathrm{Bernoulli}(\frac{1}{2})$
        \If{$X=1$}
        \State Sample $a_0\sim \pi_{\theta}(s_0)$
        \EndIf
        \For{$j\in\{0,\cdots, T-1\}$}
            \State Execute $a_{j}$ and observe $s_{j+1}\sim P(s_j, a_j)$
            \State Sample $a_{j+1}\sim \pi_{\theta}(s_{j+1})$
        \EndFor
        \vspace{0.2cm}
        \State $\hat{Q}^{\pi_{\theta}}(\hat{s}, \hat{a})\gets 2(1-X)\sum_{j=0}^{T}r(s_j, a_j)$
        \State $\hat{V}^{\pi_{\theta}}(\hat{s})\gets 2X\sum_{j=0}^{T}r(s_j, a_j)$
        \State $\hat{A}^{\pi_{\theta}}(\hat{s}, \hat{a})\gets \hat{Q}^{\pi_{\theta}}(\hat{s}, \hat{a})-\hat{V}^{\pi_{\theta}}(\hat{s})$
        \vspace{0.2cm}
        \State $\hat{F}_{\rho}(\theta)\gets \nabla_{\theta}\log\pi_{\theta}(\hat{a}|\hat{s})\otimes\nabla_{\theta}\log\pi_{\theta}(\hat{a}|\hat{s})$
        \State $\hat{H}_{\rho}(\theta)\gets \hat{A}^{\pi_\theta}(\hat{s}, \hat{a})\nabla_{\theta}\log \pi_\theta(\hat{a}|\hat{s})$
        \vspace{0.2cm}
        \State \textbf{Output:} $\hat{F}_{\rho}(\theta)\omega -\dfrac{1}{1-\gamma}\hat{H}_{\rho}(\theta)$
        \label{algo_sampling}
    \end{algorithmic}
\end{algorithm}

\begin{lemma}
    \label{lemma:unbiased_estimate}
    If $\hat{\nabla}_\omega L_{\nu^{\pi_\theta}_\rho}(\omega, \theta)$ denotes the gradient estimate yielded by Algorithm \ref{algo_sampling}, then the following holds.
    \begin{align}
        \mathbf{E}\left[\hat{\nabla}_\omega L_{\nu^{\pi_\theta}_\rho}(\omega, \theta)\big| \omega, \theta\right] = \nabla_\omega L_{\nu^{\pi_\theta}_\rho}(\omega, \theta)
    \end{align}
\end{lemma}
Thanks to Algorithm \ref{algo_sampling}, we can utilize a gradient-based iterative process to minimize $L_{\nu^{\pi_\theta}_\rho}(\cdot, \theta_k)$ to generate an estimate of $\omega_k^*$. Here we use the accelerated stochastic gradient descent (ASGD) process \citep{jain2018accelerating}, to achieve this goal. This is in contrast with the SGD-based algorithms typically used in the NPG literature \citep{liu2020improved}. The entire procedure is described in Algorithm \ref{algo_npg}.

\begin{algorithm}[t]
    \caption{Accelerated Natural Policy Gradient}
    \begin{algorithmic}[1]
        \State \textbf{Input:} Policy Parameter $\theta_0$, State distribution $\rho$,
        Run-time Parameters $K, H$, Learning Parameters $(\eta, \alpha, \beta, \xi, \delta)$
        \vspace{0.2cm}
        \For{$k\in\{0, \cdots, K-1\}$}
        \Comment{Outer Loop}
        \State $\mathbf{x}_0, \mathbf{v}_0\gets \mathbf{0}$
        \vspace{0.2cm}
        \For{$h\in\{0, \cdots, H-1\}$} \Comment{Inner Loop}
        \State \Comment{Accelerated Stochastic Gradient Descent}
        \begin{align}
            \label{eq:asgd_1}
            & \mathbf{y}_h \gets \alpha\mathbf{x}_{h}+(1-\alpha)\mathbf{v}_h
        \end{align}
        \State Get $\hat{\nabla}_{\omega} L_{\nu^{\pi_\theta}_\rho}(\omega,\theta_k)\big|_{\omega=\mathbf{y}_h}$ via Algorithm \ref{algo_sampling}
        \begin{align}
        \label{eq:asgd_2}
            &\mathbf{x}_{h+1}\gets \mathbf{y}_h - \delta \hat{\nabla}_{\omega} L_{\nu^{\pi_\theta}_\rho}(\omega,\theta_k)\big|_{\omega=\mathbf{y}_h}\\
            \label{eq:asgd_3}
            & \mathbf{z}_h \gets \beta \mathbf{y}_h + (1-\beta) \mathbf{v}_h\\
            \label{eq:asgd_4}
            & \mathbf{v}_{h+1}\gets \mathbf{z}_h - \xi \hat{\nabla}_{\omega} L_{\nu^{\pi_\theta}_\rho}(\omega,\theta_k)\big|_{\omega=\mathbf{y}_h}
        \end{align}
        \EndFor
        \State Tail Averaging:
        \begin{align}
            \omega_k\gets \dfrac{2}{H}\sum_{\frac{H}{2}<h\leq H} \mathbf{x}_h
            \label{eq:tail_average}
        \end{align}
        \State Policy Parameter Update:
        \begin{align}
            \theta_{k+1}\gets \theta_k + \eta \omega_k
            \label{eq:policy_par_update}
        \end{align}
        \EndFor
        \State \textbf{Output:} $\{\theta_k\}_{k=0}^{K-1}$
    \end{algorithmic}
    \label{algo_npg}
\end{algorithm}

Algorithm \ref{algo_npg} can be broadly segregated into two parts. In its \textit{outer loop}, the policy parameters are updated $K$ number of times following $\eqref{eq:policy_par_update}$ where $\{\omega_k\}_{k=0}^{K-1}$ denote the estimates of $\{\omega_k^*\}_{k=0}^{K-1}$. These estimates are calculated by iterating the ASGD algorithm $H$ number of times in the \textit{inner loop}. Each ASGD iteration consists of four steps described by $(\ref{eq:asgd_1})-(\ref{eq:asgd_4})$ where $(\alpha,\beta,\xi,\delta)$ are the associated learning rates. Note that in an SGD setup, the function argument is updated based only on the gradient observed at its current value. In contrast, the convoluted ASGD steps $(\ref{eq:asgd_1})-(\ref{eq:asgd_4})$ ensure that the gradients obtained in the past iterations are also taken into account while updating the argument. The idea is essentially similar to Nesterov's acceleration algorithm for deterministic gradients \citep{nesterov2012efficiency}. The estimate, $\omega_k$, is finally given by $\eqref{eq:tail_average}$ where the outcomes of the last $H/2$ iterations are averaged.
\section{Sample Complexity Analysis}
\label{sec:samle_compexity}

In this section, we shall discuss the convergence properties of Algorithm \ref{algo_npg}. Before delving into the mathematical details, we would like to state the assumptions that are needed for the analysis.

\begin{assumption}\label{ass_score}
	The score function is $G$-Lipschitz and $B$-smooth. Mathematically, the following relations hold $\forall \theta, \theta_1,\theta_2 \in\mathbb{R}^{\mathrm{d}},\forall (s,a)\in\mathcal{S}\times\mathcal{A}$.
	\begin{equation}
		\begin{aligned}
			(a)~&\Vert \nabla_\theta\log\pi_\theta(a\vert s)\Vert\leq G\\
			(b)~&\Vert \nabla_\theta\log\pi_{\theta_1}(a\vert s)-\nabla_\theta\log\pi_{\theta_2}(a\vert s)\Vert\leq B\Vert \theta_1-\theta_2\Vert\quad
		\end{aligned}
	\end{equation}
 where $B$ and $G$ are some positive reals.
\end{assumption}

The Lipschitz continuity and smoothness of the score function are common assumptions in the literature and these can be verified for simple parameterized policies such as Gaussian policies \citep{liu2020improved, Mengdi2021, Alekh2020}.

Note that Assumption \ref{ass_score} implies the following result.
\begin{lemma}
\label{lemma_2}
    If Assumption \ref{ass_score} holds, then $J_{\rho}(\cdot)$ defined in $\eqref{eq:max_J}$ satisfies the following properties $\forall \theta\in\mathbb{R}^{\mathrm{d}}$.
    \begin{align*}
        &(a)~ \Vert \nabla_{\theta} J_\rho(\theta) \Vert \leq \dfrac{G}{(1-\gamma)^2}\\
        &(b)~ J_{\rho}(\cdot) \text{~is~}L-\text{smooth,~} L\triangleq\dfrac{B}{(1-\gamma)^2}+\dfrac{2G^2}{(1-\gamma)^3}
    \end{align*}
\end{lemma}

\begin{proof}
    Statement $(a)$ can be proven by combining Assumption \ref{ass_score}(a) with $\eqref{eq:pg_expression}$ whereas statement $(b)$ follows directly from Proposition 4.2 of \citep{xu2019sample}.
\end{proof}
Lemma \ref{lemma_2} establishes that the function $J_{\rho}(\cdot)$ is Lipschitz continuous and smooth with appropriate parameters. These two properties will be useful in our further analysis.

\begin{assumption}
    \label{ass_epsilon_bias}
    The compatible function approximation error defined in $\eqref{eq:def_app_error}$ satisfies the following $\forall \theta\in\mathbb{R}^{\mathrm{d}}$.
    \begin{align}
    \label{eq:eq_18}
        L_{\nu^{\pi^*}_{\rho}}(\omega^*_{\theta}, \theta) \leq \epsilon_{\mathrm{bias}}
    \end{align}
    where $\pi^*$ is the optimal policy i.e., $\pi^* = \arg\max_{\pi} J^{\pi}_\rho$ and $\omega_{\theta}^*$ is defined as follows.
    \begin{align}
        \omega^*_{\theta} \triangleq {\arg\max}_{\omega\in\mathbb{R}^{\mathrm{d}}} L_{\nu^{\pi_\theta}_{\rho}}(\omega, \theta)
    \end{align}
\end{assumption}

Assumption \ref{ass_epsilon_bias} dictates that the parameterized policy class is rich enough such that for each policy parameter, $\theta$, the transferred compatible function approximation error, symbolized by the LHS of $(\ref{eq:eq_18})$, is bounded by a certain positive number, $\epsilon_{\mathrm{bias}}$. One can show that $\epsilon_{\mathrm{bias}}=0$ for softmax parameterization \citep{agarwal2021theory}. Furthermore, $\epsilon_{\mathrm{bias}}$ turns out to be small for rich neural network-based parameterization \citep{wang2019neural}.

 \begin{assumption}
 \label{ass_4}
     There exists a constant $\mu_F>0$ such that $F_{\rho}(\theta)-\mu_F I_{\mathrm{d}}$ is positive semidefinite where $I_{\mathrm{d}}$ denotes an identity matrix of dimension $\mathrm{d}$ and $F_\rho(\theta)$ is defined in $\eqref{eq:def_F_rho_theta}$. Equivalently, $F_{\rho}(\theta)\succcurlyeq\mu_F I_{\mathrm{d}}$. 
 \end{assumption}

 Assumption \ref{ass_4} states that the Fisher information function cannot be too small. This is also commonly used in the policy gradient-based literature \citep{liu2020improved, bai2023achieving, zhang2020global}. This is obeyed by Gaussian policies with linearly parameterized means. The set of policies that follows Assumption \ref{ass_4} is called a Fisher Non-degeneracy (FND) set. Note that, $F_\rho(\theta)=$ $\nabla^2_{\omega} L_{\nu^{\pi_\theta}_\rho}(\omega, \theta)$. Hence, Assumption \ref{ass_4} essentially states that  $L_{\nu^{\pi_\theta}_\rho}(\cdot, \theta)$ is  strongly convex with parameter, $\mu_F$.  \citep{mondal2023mean} lays out a concrete example of a class of policies that satisfies assumptions $\ref{ass_score}-\ref{ass_4}$. % {\bf mention the class here and refer paper for details. }

\subsection{Outer Loop Analysis}

Recall that in the outer loop, the policy parameter $\theta_k$'s are updated via $\eqref{eq:policy_par_update}$ where $\omega_k$'s are estimates of $\omega_k^*$'s, the natural policy gradients. The following lemma establishes how $\theta_k$'s convergence is intimately connected with the convergence of $\omega_k$'s.

\begin{lemma}
    \label{lemma:local_global}
    Let, the parameters $\{\theta_k\}_{k=0}^{K-1}$ be updated via $\eqref{eq:policy_par_update}$, $\pi^*$ be the optimal policy and $J_{\rho}^*$ be the optimal value of the function $J_{\rho}(\cdot)$. If assumption \ref{ass_score} and \ref{ass_epsilon_bias} hold, then the following inequality is satisfied. 
    	\begin{equation}\label{eq:general_bound}
			\begin{split}
			&J_{\rho}^{*}-\frac{1}{K}\sum_{k=0}^{K-1}\mathbf{E}[J_{\rho}(\theta_k)]\leq \sqrt{\epsilon_{\mathrm{bias}}}\\
            &+\frac{G}{K}\sum_{k=0}^{K-1}\mathbf{E}\Vert(\mathbf{E}\left[\omega_k|\theta_k\right]-\omega^*_k)\Vert
			+\frac{B\eta}{2K}\sum_{k=0}^{K-1}\mathbf{E}\Vert\omega_k\Vert^2\\
            &+\frac{1}{\eta K}\mathbf{E}_{s\sim d_\rho^{\pi^*}}[KL(\pi^*(\cdot\vert s)\Vert\pi_{\theta_0}(\cdot\vert s))]		\end{split}
		\end{equation}
  where $KL(\cdot||\cdot)$ is the Kullback-Leibler divergence.
\end{lemma}

Lemma \ref{lemma:local_global} is almost identical to Proposition 4.5 of \citep{liu2020improved}. However, there is an important difference between these results. The first order term in the cited paper is $\mathbf{E}\Vert\omega_k-\omega_k^*\Vert$ whereas in $\eqref{eq:general_bound}$, we have improved it to $\mathbf{E}\Vert (\mathbf{E}[\omega_k|\theta_k]-\omega_k^*)\Vert$. Such a seemingly innocuous change turns out to be crucial to improving the sample complexity. 
\begin{proof}
(Outline) To get an idea of how the first order term is improved, note that we can write the following inequality using the smoothness property of the score function (Assumption \ref{ass_score}) and the update rule $\eqref{eq:policy_par_update}$.
\begin{equation}
    \begin{aligned}
        &\mathbf{E}_{s\sim d^{\pi_\rho^*}}\left[KL(\pi^*(\cdot\vert s)\Vert\pi_{\theta_k}(\cdot\vert s))-KL(\pi^*(\cdot\vert s)\Vert\pi_{\theta_{k+1}}(\cdot\vert s))\right]\\
       &\leq \eta\mathbf{E}_{(s, a)\sim \nu^{\pi^*}_\rho}\left[\nabla_\theta\log\pi_{\theta_k}(a\vert s)\cdot\omega^*_k\right]-\frac{B\eta^2}{2}\Vert\omega_k\Vert^2\\
        &\hspace{0.3cm} + \eta\mathbf{E}_{(s, a)\sim \nu^{\pi^*}_\rho}\left[\nabla_\theta\log\pi_{\theta_k}(a\vert s)\cdot(\omega_k-\omega^*_k)\right]
    \end{aligned}
    \label{eq:eq_20_}
\end{equation}

The first term on the RHS of $\eqref{eq:eq_20_}$ can be connected to $\epsilon_{\mathrm{bias}}$ via Assumption \ref{ass_epsilon_bias} and the performance difference lemma \citep{agarwal2021theory} stated as,
\begin{align*}
    J_\rho^* - J_\rho(\theta_k) = \frac{1}{1-\gamma}\mathbf{E}_{(s, a)\sim \nu^{\pi^*}_\rho}\left[A^{\pi_{\theta_k}}(s, a)\right]
\end{align*}
Note that the second term on the RHS of $\eqref{eq:eq_20_}$ is connected to the second-order error in $\eqref{eq:general_bound}$. Before bounding the third term, we take expectations on  both sides of $\eqref{eq:eq_20_}$. Then a bound is obtained as follows.
\begin{align*}
    &\Big|\mathbf{E}_{(s, a)\sim \nu^{\pi^*}_\rho}\mathbf{E}\left[\nabla_\theta\log\pi_{\theta_k}(a\vert s)\cdot(\omega_k-\omega^*_k)\right]\Big|\\
    &= \Big|\mathbf{E}_{(s, a)\sim \nu^{\pi^*}_\rho}\mathbf{E}\left[\nabla_\theta\log\pi_{\theta_k}(a\vert s)\cdot(\mathbf{E}[\omega_k|\theta_k]-\omega^*_k)\right]\Big|\\
    &\overset{(a)}{\leq} G\times\mathbf{E}\Vert(\mathbf{E}[\omega_k|\theta_k]-\omega^*_k)\Vert
\end{align*}
where $(a)$ follows from Assumption \ref{ass_score}. Rearranging the terms and taking a sum over $k\in\{0, 1, \cdots, K-1\}$, we finally prove Lemma \ref{lemma:local_global}. 
\end{proof}

Observe that, if the estimate $\omega_k$ produced by the inner loop is a good approximation of $\omega_k^*$, then the first order term in $(\ref{eq:general_bound})$ will be small. To bound the second-order term, note the following inequality.
\begin{align}
\label{eq:eq_21}
    \begin{split}
        &\dfrac{1}{K}\sum_{k=0}^{K-1}\mathbf{E}\Vert \omega_k \Vert^2\\
        &\leq \dfrac{1}{K}\sum_{k=0}^{K-1}\mathbf{E}\Vert \omega_k -\omega_k^*\Vert^2 + \dfrac{1}{K}\sum_{k=0}^{K-1}\mathbf{E}\Vert \omega_k^* \Vert^2\\
        &\overset{(a)}{\leq} \dfrac{1}{K}\sum_{k=0}^{K-1}\mathbf{E}\Vert \omega_k -\omega_k^*\Vert^2 + \dfrac{\mu_F^{-2}}{ K}\sum_{k=0}^{K-1}\mathbf{E}\Vert \nabla_{\theta} J_\rho(\theta_k) \Vert^2
    \end{split}
\end{align}
where $(a)$ follows from Assumption \ref{ass_4} and the definition that $\omega_k^*=F_\rho(\theta_k)^{\dagger}\nabla_\theta J_\rho(\theta_k)$. Notice that the first term in $(\ref{eq:eq_21})$ will be small if $\omega_k$ well approximates $\omega_k^*$, $\forall k$. To bound the second term, we use the following lemma.

\begin{lemma}
    \label{lemma_gradient_bound}
    Let $\{\theta_k\}_{k=0}^{K-1}$ be recursively defined by $\eqref{eq:policy_par_update}$. If Assumption \ref{ass_score} and \ref{ass_4} hold, then the following inequality is satisfied with learning rate $\eta = \frac{\mu_F^2}{4G^2 L}$.
    \begin{align}
    \begin{split}
        \dfrac{1}{K}\sum_{k=0}^{K-1}&\mathbf{E}\Vert \nabla_\theta J_\rho(\theta_k)\Vert^2 \leq \left(\dfrac{8G^4 L}{\mu_F^2(1-\gamma)}\right)\dfrac{1}{K}\\
        &+\left(2G^4+\mu_F^2\right)\left(\dfrac{1}{K}\sum_{k=0}^{K-1}\mathbf{E}\Vert\omega_k-\omega_k^*\Vert^2\right)
    \end{split}
    \end{align}
\end{lemma}

Lemma \ref{lemma_gradient_bound} reveals that the second term in $\eqref{eq:eq_21}$ can be bounded above by the approximation error of the inner loop. Combining Lemma \ref{lemma:local_global}, $\eqref{eq:eq_21}$, and Lemma \ref{lemma_gradient_bound}, we get the following result.
\begin{corollary}
    \label{corr_1}
    Consider the setup described in Lemma \ref{lemma:local_global}. If assumptions \ref{ass_score}-\ref{ass_4} hold, then the following inequality is satisfied for $\eta=\frac{\mu_F^2}{4G^2 L}$.
    \begin{equation}\label{eq:general_bound_corr}
		\begin{split}
			&J_{\rho}^{*}-\frac{1}{K}\sum_{k=0}^{K-1}\mathbf{E}[J_{\rho}(\theta_k)]\\
            &\leq \sqrt{\epsilon_{\mathrm{bias}}}
            +\frac{G}{K}\sum_{k=0}^{K-1}\mathbf{E}\Vert(\mathbf{E}\left[\omega_k|\theta_k\right]-\omega^*_k)\Vert\\
			&+\dfrac{B}{4L}\left(\dfrac{\mu_F^2}{G^2}+G^2\right)\left(\dfrac{1}{K}\sum_{k=0}^{K-1}\mathbf{E}\Vert\omega_k-\omega_k^*\Vert^2\right)\\
            &+\dfrac{G^2}{\mu_F^2K}\left(\dfrac{B}{1-\gamma}+4L\mathbf{E}_{s\sim d_\rho^{\pi^*}}[KL(\pi^*(\cdot\vert s)\Vert\pi_{\theta_0}(\cdot\vert s))]\right)
        \end{split}
	\end{equation}
\end{corollary}

Observe that both the first-order approximation error, $\mathbf{E}\Vert(\mathbf{E}[\omega_k|\theta]-\omega_k^*)\Vert$ and the second-order error, $\mathbf{E}\Vert\omega_k-\omega_k^*\Vert^2$ must be functions of $H$, the number of iterations of the inner loop. In section \ref{sec_inner_loop}, we point out the exact nature of this dependence. Furthermore, the last term in $\eqref{eq:general_bound_corr}$ is a function of $K$, the length of the outer loop. Section \ref{sec_final_result} discusses how $H$ and $K$ can be judiciously chosen to achieve the desired sample complexity. Note that due to the presence of $\epsilon_{\mathrm{bias}}$ in $(\ref{eq:general_bound_corr})$, the optimality gap of $J_{\rho}(\cdot)$ cannot be made arbitrarily small. This is an unavoidable error for any parameterized policy class that is incomplete i.e., does not contain all stochastic policies. However, for rich classes of policies, this error can be assumed to be insignificant.

\subsection{Inner Loop Analysis}
\label{sec_inner_loop}

Recall that the inner loop produces estimates of $\omega_k^*=\arg\min_{\omega\in\mathbb{R}^{\mathrm{d}}}L_{\nu_{\rho}^{\pi_{\theta}}}(\omega, \theta)|_{\theta=\theta_k}$ using stochastic gradients obtained from Algorithm \ref{algo_sampling}. Lemma \ref{lemma:unbiased_estimate} showed that the gradient estimates are unbiased. The following lemma characterizes their variance.
\begin{lemma}
    \label{lemma_noise_variance}
    Let $\hat{\nabla}_{\omega}L_{\nu^{\pi_\theta}_\rho}(\omega, \theta)$ denote the gradient estimate generated by Algorithm \ref{algo_sampling}. If assumptions \ref{ass_score} and \ref{ass_4} hold, the following semidefinite inequality is satisfied for any $\theta\in\mathbb{R}^{\mathrm{d}}$.
    \begin{align}
    \label{eq:eq_24}
        \begin{split}
            \mathbf{E}&\left[\hat{\nabla}_{\omega}L_{\nu^{\pi_\theta}_\rho}(\omega_{\theta}^*, \theta)\otimes\hat{\nabla}_{\omega}L_{\nu^{\pi_\theta}_\rho}(\omega_{\theta}^*, \theta)\right]\preccurlyeq \sigma^2 F_\rho(\theta)
        \end{split}
    \end{align}
    where $\omega_\theta^* = \arg\min_{\omega\in\mathbb{R}^{\mathrm{d}}}L_{\nu^{\pi_\theta}_{\rho}}(\omega, \theta)$, $F_\rho(\theta)$ is given by $\eqref{eq:def_F_rho_theta}$, and $\sigma^2$ is defined by the following expression.
    \begin{align}
    \label{eq:def_sigma_sq}
        \sigma^2 \triangleq \dfrac{2G^4}{\mu_F^2(1-\gamma)^4}+\dfrac{32}{(1-\gamma)^4}
    \end{align}
\end{lemma}

Note that the term $\sigma^2$ can be interpreted as the scaled variance of the gradient estimates. In a noiseless scenario (deterministic estimates), the LHS of $(\ref{eq:eq_24})$ turns out to be zero which allows us to choose $\sigma^2=0$. This information will be useful in proving one of our subsequent results. Merging Lemma \ref{lemma_noise_variance} with the convergence result of ASGD as stated in \citep[Corollary 2]{jain2018accelerating}, we arrive at the following lemma.

\begin{lemma}
    \label{lemma_second_order} Let $\omega_k$ be an estimate of $\omega_k^*$ generated by Algorithm \ref{algo_npg} at the $k$th iteration of the outer loop, $k\in\{0, 1, \cdots, K-1\}$. If assumptions \ref{ass_score} and \ref{ass_4} hold, then the following inequality is satisfied with learning rates $\alpha = \frac{3\sqrt{5}G^2}{\mu_F+3\sqrt{5}G^2}$, $\beta=\frac{\mu_F}{9G^2}$, $\xi=\frac{1}{3\sqrt{5}G^2}$, and $\delta=\frac{1}{5G^2}$ provided that the inner loop length obeys $H>\bar{C}\frac{G^2}{\mu_F}\log\left(\sqrt{\mathrm{d}}\frac{G^2}{\mu_F}\right)$ for some universal constant, $\bar{C}$.
    \begin{align}
    \label{eq:eq_second_order_bound_lemma}
    \begin{split}
        \mathbf{E}\Vert\omega_k - \omega_k^*\Vert^2 &\leq 22\dfrac{\sigma^2\mathrm{d}}{\mu_F H}\\
        +& C\exp\left(-\dfrac{\mu_F }{20G^2}H\right)\left[\dfrac{1}{\mu_F(1-\gamma)^4}\right]
    \end{split}
    \end{align}
    where $C$ denotes a universal constant and $\sigma^2$ is defined in $\eqref{eq:def_sigma_sq}$.
\end{lemma}

Lemma \ref{lemma_second_order} implies that the second-order inner loop approximation error can be bounded above as $\mathcal{O}\left(1/H\right)$. Interestingly, however, the leading factor $1/H$ appears due to the presence of noise variance, $\sigma^2$. Hence, in a noiseless scenario ($\sigma^2=0$), exponentially fast convergence should be viable. This idea has been exploited in the following result.

\begin{lemma}
    \label{lemma_first_order}
    Consider the same setup and the choice of parameters as dictated in Lemma \ref{lemma_second_order}. If assumptions \ref{ass_score} and \ref{ass_4} hold, then we have the following inequality.
    \begin{align}
    \begin{split}
            \mathbf{E}\Vert(\mathbf{E}[&\omega_k|\theta_k] -\omega_k^*)\Vert \\
            &\leq \sqrt{C}\exp\left(-\dfrac{\mu_F }{40G^2}H\right)\left(\dfrac{1}{\sqrt{\mu_F}(1-\gamma)^2}\right)
    \end{split}
    \end{align}
\end{lemma}

Lemma \ref{lemma_first_order} dictates that the first-order inner loop approximation error exponentially decreases with $H$. To intuitively understand the proof of this result, consider the ASGD update rules $\eqref{eq:asgd_1}-\eqref{eq:asgd_4}$ in Algorithm \ref{algo_npg} that are used to determine $\omega_k$ via $\eqref{eq:tail_average}$, $ \forall k\in\{0,  \cdots, K-1\}$. Note the following property of the gradient estimator produced by Algorithm \ref{algo_sampling}.
\begin{align}
\label{eq:grad_property_inner_loop}
\begin{split}
    &\mathbf{E}\left[\hat{\nabla}_\omega L_{\nu^{\pi_\theta}_\rho}(\omega, \theta)\big|\theta\right]\\
    &\overset{}{=} \mathbf{E}\left[\mathbf{E}\left[\hat{\nabla}_\omega L_{\nu^{\pi_\theta}_\rho}(\omega, \theta)\big|\omega, \theta\right]\big|\theta\right]\\
    &\overset{(a)}{=} \mathbf{E}\left[\nabla_\omega L_{\nu^{\pi_\theta}_\rho}(\omega, \theta)\big|\theta\right]\overset{(b)}{=}\nabla_\omega L_{\nu^{\pi_\theta}_\rho}\left(\mathbf{E}\left[\omega |\theta\right], \theta\right)
\end{split}
\end{align}
where $(a)$ uses the unbiasedness (Lemma \ref{lemma:unbiased_estimate}) of the estimate and $(b)$ follows from the linearity of the gradient. Taking expectation conditioned on $\theta_k$ on both sides of $\eqref{eq:tail_average}$, we obtain the following relation.
\begin{align}
    \label{eq:noiseless_tail_average}\mathbf{E}\left[\omega_k|\theta_k\right] = \dfrac{2}{H}\sum_{\frac{H}{2}<h\leq H} \mathbf{E}\left[\mathbf{x}_h\big|\theta_k\right]
\end{align}
where the conditional expectation of $\{\mathbf{x}_h\}_{h=1}^H$ are recursively computed as follows starting with the initial condition, $\bar{\mathbf{x}}_0\triangleq \mathbf{E}[\mathbf{x}_0|\theta_k]=\mathbf{0}$ and $\bar{\mathbf{v}}_0\triangleq\mathbf{E}[\mathbf{v}_0|\theta_k]=\mathbf{0}$.
\begin{align}
     \label{eq:asgd_1_conditioned}
        & \bar{\mathbf{y}}_h = \alpha\bar{\mathbf{x}}_{h}+(1-\alpha)\bar{\mathbf{v}}_h \\
    \label{eq:asgd_2_conditioned}
        &\bar{\mathbf{x}}_{h+1}=\bar{\mathbf{y}}_h - \delta \nabla_{\omega} L_{\nu^{\pi_\theta}_\rho}(\omega,\theta_k)\big|_{\omega=\bar{\mathbf{y}}_h}\\
        \label{eq:asgd_3_conditioned}
        & \bar{\mathbf{z}}_h = \beta \bar{\mathbf{y}}_h + (1-\beta) \bar{\mathbf{v}}_h\\
        \label{eq:asgd_4_conditioned}
        & \bar{\mathbf{v}}_{h+1}= \bar{\mathbf{z}}_h - \xi \nabla_{\omega} L_{\nu^{\pi_\theta}_\rho}(\omega,\theta_k)\big|_{\omega=\bar{\mathbf{y}}_h}
\end{align}
where $\bar{\mathbf{x}}_h\triangleq\mathbf{E}\left[\mathbf{x}_h|\theta_k\right]$, $\bar{\mathbf{y}}_h\triangleq\mathbf{E}\left[\mathbf{y}_h|\theta_k\right]$, $\bar{\mathbf{v}}_h\triangleq\mathbf{E}\left[\mathbf{v}_h|\theta_k\right]$, $\bar{\mathbf{z}}_h\triangleq\mathbf{E}\left[\mathbf{z}_h|\theta_k\right]$ and $h\in\{0, 1, \cdots, H\}$. The above equations can be derived by taking conditional expectation on both sides of $\eqref{eq:asgd_1}$-$\eqref{eq:asgd_4}$ and invoking $\eqref{eq:grad_property_inner_loop}$. Note that  $\eqref{eq:noiseless_tail_average}$ together with $\eqref{eq:asgd_1_conditioned}-\eqref{eq:asgd_4_conditioned}$ can be interpreted as the update rules of a noiseless (deterministic) accelerated gradient descent procedure. This allows us to derive a bound similar to $\eqref{eq:eq_second_order_bound_lemma}$ but with $\sigma^2=0$.

\subsection{Final Result}
\label{sec_final_result}

Lemma \ref{lemma_second_order} and \ref{lemma_first_order} together dictate that the second and third terms in $\eqref{eq:general_bound_corr}$ can be bounded above as $\mathcal{O}(1/H)$. Additionally, the last term in $\eqref{eq:general_bound_corr}$ is  $\mathcal{O}(1/K)$. Thus by taking $H=\mathcal{O}(1/\epsilon)$ and $K=\mathcal{O}(1/\epsilon)$, we  guarantee the optimality gap to be at most $\sqrt{\epsilon_{\mathrm{bias}}}+\epsilon$. This results in a sample complexity of $\mathcal{O}(HK)=\mathcal{O}(1/\epsilon^2)$ and an iteration complexity of $\mathcal{O}(K)=\mathcal{O}(1/\epsilon)$. The result is formalized in the following theorem.
\begin{theorem}
    \label{theorem_final}
    Let $\{\theta_k\}_{k=0}^{K-1}$ be the policy parameters generated by Algorithm \ref{algo_npg}, $\pi^*$ be the optimal policy and $J^*_\rho$ denote the optimal value of $J_\rho(\cdot)$ corresponding to an initial distribution $\rho$. Assume that assumptions $\ref{ass_score}-\ref{ass_4}$ hold and the learning parameters are chosen as stated in Corollary \ref{corr_1} and Lemma \ref{lemma_second_order}. For all sufficiently small $\epsilon$, and  $H=\mathcal{O}(1/\epsilon)$, $K=\mathcal{O}(1/\epsilon)$\footnote{The exact values of these parameters are stated in the proof of the theorem in the appendix.}, the following holds.
    \begin{align}
    \label{eq:theorem_main_result}
        J_\rho^*-\dfrac{1}{K}\sum_{k=0}^{K-1}\mathbf{E}[J_\rho(\theta_k)] \leq \sqrt{\epsilon_{\mathrm{bias}}}+\epsilon
    \end{align}
    This results in $\mathcal{O}((1-\gamma)^{-6}\epsilon^{-2})$ sample complexity and $\mathcal{O}((1-\gamma)^{-3}\epsilon^{-1})$ iteration complexity.
\end{theorem}

A few remarks are in order. Note the importance of the first-order approximation error, $\mathbf{E}\Vert(\mathbf{E}[\omega_k|\theta_k]-\omega_k^*)\Vert$,   in establishing the sample complexity result. If we follow existing analysis such as that given in \citep{liu2020improved}, then this term would turn out to be $\mathbf{E}\Vert\omega_k-\omega_k^*\Vert$ which is bounded above as $\mathcal{O}({H}^{-\frac{1}{2}})$ following Lemma \ref{lemma_second_order}. This requires  $H=\mathcal{O}(\epsilon^{-2})$ and $K=\mathcal{O}(\epsilon^{-1})$ to guarantee an optimality gap of $\sqrt{\epsilon_{\mathrm{bias}}}+\epsilon$. Unfortunately, the sample complexity, in this case, worsens to $\mathcal{O}(\epsilon^{-3})$.

Secondly, note the importance of ASGD in our analysis. If we use SGD (instead of ASGD) in the inner loop and utilize its convergence result such as that provided in \citep{bach2013non}, the first-order error still ends up being $\mathcal{O}(H^{-\frac{1}{2}})$. With an improved analysis of SGD \citep{jain2016parallelizing}, one can make the first-order term an exponentially decreasing function of $H$.
However, it requires the learner to have an exact knowledge of the occupancy measure, $\nu^{\pi_{\theta_k}}_\rho$ in order to choose an appropriate learning rate. This directly contrasts with the basic premise of reinforcement learning.

\section{Conclusion}

This paper proposes the Accelerated Natural Policy Gradient (ANPG) algorithm that uses an acceleration-based stochastic gradient descent procedure for finding natural gradients. By introducing an improved global convergence analysis, and utilizing a novel observation that a first-order term can be interpreted as the error generated by a noiseless accelerated gradient descent procedure, we establish a sample complexity of $\mathcal{O}(\epsilon^{-2})$ and an iteration complexity of $\mathcal{O}(\epsilon^{-1})$ for the ANPG. This improves the SOTA sample complexity by a factor of $\log\left(\frac{1}{\epsilon}\right)$. Within the class of all Hessian-free and importance sample-free algorithms, it beats the SOTA sample complexity by a $\mathcal{O}(\epsilon^{-\frac{1}{2}})$ factor and matches the best known iteration complexity.

This paper unleashes several promising directions for future research. For example, our analysis can be used to improve the sample complexities in constrained RL, concave utility RL, etc. Another important task would be improving the iteration complexity while maintaining the sample complexity reported in our paper.

\bibliographystyle{abbrvnat}

%\newpage
%\input{Sections/Checklist}

%\if 0
\clearpage
\appendix
\onecolumn

\section{Proof of Lemma \ref{lemma:unbiased_estimate}}

Note that the following relation holds $\forall (s, a)\in \mathcal{S}\times \mathcal{A}$.
\begin{align}
\label{eq:eq_appndx_lemma_1_proof_1}
    \mathrm{Pr}(\hat{s}=s, \hat{a}=a|\rho, \pi_\theta)=(1-\gamma)\sum_{t=0}^\infty \gamma^t \mathrm{Pr}(s_t=s, a_t=a|s_0\sim \rho, \pi_\theta) = \nu^{\pi_\theta}_\rho(s, a)
\end{align}

Moreover, for a given pair $(s, a)\in \mathcal{S}\times \mathcal{A}$, the following equality holds.
\begin{align}
\label{eq:eq_appndx_lemma_1_proof_2}
\begin{split}
    \mathbf{E}\left[\hat{Q}^{\pi_\theta}(s, a)\right] &= \dfrac{1}{2}\times 2\mathbf{E}\left[(1-\gamma)\sum_{t=0}^{\infty}\gamma^t \sum_{j=0}^t r(s_j, a_j)\bigg| s_0=s, a_0=a, \pi_\theta \right]\\
    &\overset{(a)}{=} \mathbf{E}\left[(1-\gamma)\sum_{j=0}^{\infty}r(s_j, a_j) \sum_{t=j}^\infty \gamma^t \bigg| s_0=s, a_0=a, \pi_\theta \right]\\
    &\overset{}{=} \mathbf{E}\left[\sum_{j=0}^{\infty}\gamma^j r(s_j, a_j)   \bigg| s_0=s, a_0=a, \pi_\theta \right] = Q^{\pi_\theta}(s, a)
\end{split}
\end{align}
where $(a)$ is obtained by exchanging the order of the sum. Similarly, one can exhibit that $\mathbf{E}\left[\hat{V}^{\pi_\theta}(s)\right] = V^{\pi_\theta}(s)$, and therefore $\mathbf{E}\left[\hat{A}^{\pi_\theta}(s, a)\right] = A^{\pi_\theta}(s, a)$, $\forall s\in\mathcal{S}$. Combining this result with $\eqref{eq:eq_appndx_lemma_1_proof_1}$, we establish the lemma. 

\section{Proof of Lemma \ref{lemma:local_global}}

Using the definition of KL divergence, we obtain the following.
	\begin{equation}
	\begin{aligned}
	&\mathbf{E}_{s\sim d_\rho^{\pi^*}}[KL(\pi^*(\cdot\vert s)\Vert\pi_{\theta_k}(\cdot\vert s))-KL(\pi^*(\cdot\vert s)\Vert\pi_{\theta_{k+1}}(\cdot\vert s))]\\
	&=\mathbf{E}_{s\sim d_\rho^{\pi^*}}\mathbf{E}_{a\sim\pi^*(\cdot\vert s)}\bigg[\log\frac{\pi_{\theta_{k+1}(a\vert s)}}{\pi_{\theta_k}(a\vert s)}\bigg]\\
	&\overset{(a)}\geq\mathbf{E}_{s\sim d_\rho^{\pi^*}}\mathbf{E}_{a\sim\pi^*(\cdot\vert s)}[\nabla_\theta\log\pi_{\theta_k}(a\vert s)\cdot(\theta_{k+1}-\theta_k)]-\frac{B}{2}\Vert\theta_{k+1}-\theta_k\Vert^2\\
	&=\eta\mathbf{E}_{s\sim d_\rho^{\pi^*}}\mathbf{E}_{a\sim\pi^*(\cdot\vert s)}[\nabla_{\theta}\log\pi_{\theta_k}(a\vert s)\cdot\omega_k]-\frac{B\eta^2}{2}\Vert\omega_k\Vert^2\\
	&=\eta\mathbf{E}_{s\sim d_\rho^{\pi^*}}\mathbf{E}_{a\sim\pi^*(\cdot\vert s)}[\nabla_\theta\log\pi_{\theta_k}(a\vert s)\cdot\omega^*_k]+\eta\mathbf{E}_{s\sim d_\rho^{\pi^*}}\mathbf{E}_{a\sim\pi^*(\cdot\vert s)}[\nabla_\theta\log\pi_{\theta_k}(a\vert s)\cdot(\omega_k-\omega^*_k)]-\frac{B\eta^2}{2}\Vert\omega_k\Vert^2\\
	&=\eta[J_\rho^{*}-J_\rho(\theta_k)]+\eta\mathbf{E}_{s\sim d_\rho^{\pi^*}}\mathbf{E}_{a\sim\pi^*(\cdot\vert s)}[\nabla_\theta\log\pi_{\theta_k}(a\vert s)\cdot\omega^*_k]-\eta[J_\rho^{*}-J_\rho(\theta_k)]\\
	&+\eta\mathbf{E}_{s\sim d_\rho^{\pi^*}}\mathbf{E}_{a\sim\pi^*(\cdot\vert s)}[\nabla_\theta\log\pi_{\theta_k}(a\vert s)\cdot(\omega_k-\omega^*_k)]-\frac{B\eta^2}{2}\Vert\omega_k\Vert^2\\		
    &\overset{(b)}=\eta[J_\rho^{*}-J_\rho(\theta_k)]+\eta\mathbf{E}_{s\sim d_\rho^{\pi^*}}\mathbf{E}_{a\sim\pi^*(\cdot\vert s)}\bigg[\nabla_\theta\log\pi_{\theta_k}(a\vert s)\cdot\omega^*_k-\dfrac{1}{1-\gamma}A^{\pi_{\theta_k}}(s,a)\bigg]\\
	&+\eta\mathbf{E}_{s\sim d_\rho^{\pi^*}}\mathbf{E}_{a\sim\pi^*(\cdot\vert s)}[\nabla_\theta\log\pi_{\theta_k}(a\vert s)\cdot(\omega_k-\omega^*_k)]-\frac{B\eta^2}{2}\Vert\omega_k\Vert^2\\
	&\overset{(c)}\geq\eta[J_\rho^{*}-J_\rho(\theta_k)]-\eta\sqrt{\mathbf{E}_{s\sim d_\rho^{\pi^*}}\mathbf{E}_{a\sim\pi^*(\cdot\vert s)}\bigg[\bigg(\nabla_\theta\log\pi_{\theta_k}(a\vert s)\cdot\omega^*_k-\dfrac{1}{1-\gamma}A^{\pi_{\theta_k}}(s,a)\bigg)^2\bigg]}\\
	&+\eta\mathbf{E}_{s\sim d_\rho^{\pi^*}}\mathbf{E}_{a\sim\pi^*(\cdot\vert s)}[\nabla_\theta\log\pi_{\theta_k}(a\vert s)\cdot(\omega_k-\omega^*_k)]-\frac{B\eta^2}{2}\Vert\omega_k\Vert^2\\
	&\overset{(d)}\geq\eta[J_\rho^{*}-J_\rho(\theta_k)]-\eta\sqrt{\epsilon_{\mathrm{bias}}} +\eta\mathbf{E}_{s\sim d_\rho^{\pi^*}}\mathbf{E}_{a\sim\pi^*(\cdot\vert s)}[\nabla_\theta\log\pi_{\theta_k}(a\vert s)\cdot(\omega_k-\omega^*_k)]-\frac{B\eta^2}{2}\Vert\omega_k\Vert^2\\
	\end{aligned}	
	\end{equation}
	where $(a)$ and $(b)$ are implied by Assumption \ref{ass_score}(b) and Lemma \ref{lemma_aux_1} respectively. Inequality (c) follows from the convexity of the function $f(x)=x^2$. Finally, step (d) is a consequence of Assumption \ref{ass_epsilon_bias}. Taking expectation on both sides, we obtain,
    \begin{align}
        \begin{split}
            &\mathbf{E}\left[\mathbf{E}_{s\sim d_\rho^{\pi^*}}\left[KL(\pi^*(\cdot\vert s)\Vert\pi_{\theta_k}(\cdot\vert s))-KL(\pi^*(\cdot\vert s)\Vert\pi_{\theta_{k+1}}(\cdot\vert s))\right]\right]\\
            &\geq \eta[J_\rho^{*}-\mathbf{E}\left[J_\rho(\theta_k)]\right]-\eta\sqrt{\epsilon_{\mathrm{bias}}}\\
            &\hspace{2cm}+\eta\mathbf{E}\left[\mathbf{E}_{s\sim d_\rho^{\pi^*}}\mathbf{E}_{a\sim\pi^*(\cdot\vert s)}[\nabla_\theta\log\pi_{\theta_k}(a\vert s)\cdot(\mathbf{E}[\omega_k|\theta_k]-\omega^*_k)]\right]-\frac{B\eta^2}{2}\mathbf{E}\left[\Vert\omega_k\Vert^2\right]\\
            &\overset{}{\geq } \eta[J_\rho^{*}-\mathbf{E}\left[J_\rho(\theta_k)]\right]-\eta\sqrt{\epsilon_{\mathrm{bias}}}\\
            &\hspace{2cm}-\eta\mathbf{E}\left[\mathbf{E}_{s\sim d_\rho^{\pi^*}}\mathbf{E}_{a\sim\pi^*(\cdot\vert s)}[\Vert \nabla_\theta\log\pi_{\theta_k}(a\vert s)\Vert \Vert\mathbf{E}[\omega_k|\theta_k]-\omega^*_k\Vert]\right]-\frac{B\eta^2}{2}\mathbf{E}\left[\Vert\omega_k\Vert^2\right]\\
            &\overset{(a)}{\geq } \eta[J_\rho^{*}-\mathbf{E}\left[J_\rho(\theta_k)]\right]-\eta\sqrt{\epsilon_{\mathrm{bias}}} -\eta G\mathbf{E} \Vert(\mathbf{E}[\omega_k|\theta_k]-\omega^*_k)\Vert-\frac{B\eta^2}{2}\mathbf{E}\left[\Vert\omega_k\Vert^2\right]
        \end{split}
    \end{align}
    where $(a)$ follows from Assumption \ref{ass_score}(a). Rearranging the terms, we get,
	\begin{equation}
	\begin{split}
	J_\rho^{*}-\mathbf{E}[J_\rho(\theta_k)]&\leq \sqrt{\epsilon_{\mathrm{bias}}}+ G\mathbf{E}\Vert(\mathbf{E}[\omega_k|\theta_k]-\omega^*_k)\Vert+\frac{B\eta}{2}\mathbf{E}\Vert\omega_k\Vert^2\\
	&+\frac{1}{\eta}\mathbf{E}\left[\mathbf{E}_{s\sim d_\rho^{\pi^*}}[KL(\pi^*(\cdot\vert s)\Vert\pi_{\theta_k}(\cdot\vert s))-KL(\pi^*(\cdot\vert s)\Vert\pi_{\theta_{k+1}}(\cdot\vert s))]\right]
	\end{split}
	\end{equation}
	Summing from $k=0$ to $K-1$, using the non-negativity of KL divergence and dividing the resulting expression by $K$, we get the desired result.

\section{Proof of Lemma \ref{lemma_gradient_bound}}

Using the $L$-smoothness property of the function, $J_\rho(\cdot)$ (Lemma \ref{lemma_2}), we obtain the following.
\begin{align}
\label{eq:eq_26}
\begin{split}
    &J_\rho(\theta_{k+1})\\
    &\geq J_\rho(\theta_k)+\left<\nabla_\theta J_\rho(\theta_k),\theta_{k+1}-\theta_k\right>-\frac{L}{2}\Vert\theta_{k+1}-\theta_k\Vert^2\\
    &=J_\rho(\theta_k)+\eta\left<\nabla_\theta J_\rho(\theta_k),\omega_k\right>-\frac{\eta^2 L}{2}\Vert\omega_k\Vert^2\\
    &=J_\rho(\theta_k)+\eta\left<\nabla_\theta J_\rho(\theta_k),\omega_k^*\right>+\eta\left<\nabla_\theta J_\rho(\theta_k),\omega_k-\omega_k^*\right>-\frac{\eta^2 L}{2}\Vert\omega_k-\omega_k^*+\omega_k^*\Vert^2\\
    &\overset{(a)}{\geq} J_\rho(\theta_k) +\eta\left<\nabla_\theta J_\rho(\theta_k),F_\rho(\theta_k)^{\dagger}\nabla_\theta J_\rho(\theta_k)\right> +\eta\left<\nabla_\theta J_\rho(\theta_k),\omega_k-\omega_k^*\right> -\eta^2 L \Vert \omega_k-\omega_k^*\Vert^2 - \eta^2L\Vert\omega_k^*\Vert^2\\
    &\overset{(b)}{\geq} J_\rho(\theta_k) +\dfrac{\eta}{G^2}\Vert\nabla_\theta J_\rho(\theta_k)\Vert^2 +\eta\left<\nabla_\theta J_\rho(\theta_k),\omega_k-\omega_k^*\right> -\eta^2 L \Vert \omega_k-\omega_k^*\Vert^2 - \eta^2L\Vert\omega_k^*\Vert^2\\
    &= J_\rho(\theta_k) +\dfrac{\eta}{2G^2}\Vert\nabla_\theta J_\rho(\theta_k)\Vert^2 + \dfrac{\eta}{2G^2}\left[\Vert\nabla_\theta J_\rho(\theta_k)\Vert^2 +2G^2 \left<\nabla_\theta J_\rho(\theta_k),\omega_k-\omega_k^*\right>+G^4\Vert \omega_k-\omega_k^*\Vert^2\right] \\
    &\hspace{1cm}-\left(\dfrac{\eta G^2}{2}+\eta^2 L\right) \Vert \omega_k-\omega_k^*\Vert^2 - \eta^2L\Vert\omega_k^*\Vert^2 \\
    &=J_\rho(\theta_k) +\dfrac{\eta}{2G^2}\Vert\nabla_\theta J_\rho(\theta_k)\Vert^2 + \dfrac{\eta}{2G^2}\Vert\nabla_\theta J_\rho(\theta_k)+G^2(\omega_k-\omega_k^*)\Vert^2-\left(\dfrac{\eta G^2}{2}+\eta^2 L\right) \Vert \omega_k-\omega_k^*\Vert^2 - \eta^2L\Vert\omega_k^*\Vert^2 \\
    &\geq J_\rho(\theta_k) +\dfrac{\eta}{2G^2}\Vert\nabla_\theta J_\rho(\theta_k)\Vert^2 -\left(\dfrac{\eta G^2}{2}+\eta^2 L\right) \Vert \omega_k-\omega_k^*\Vert^2 - \eta^2L\Vert F_\rho(\theta_k)^{\dagger}\nabla_\theta J_\rho(\theta_k)\Vert^2 \\
    &\overset{(c)}{\geq} J_\rho(\theta_k) +\left(\dfrac{\eta}{2G^2}-\dfrac{\eta^2 L}{\mu_F^2}\right)\Vert\nabla_\theta J_\rho(\theta_k)\Vert^2 -\left(\dfrac{\eta G^2}{2}+\eta^2 L\right) \Vert \omega_k-\omega_k^*\Vert^2  
\end{split}
\end{align}
 where $(a)$ utilizes the Cauchy-Schwarz inequality and the definition that $\omega_k^*=F_\rho(\theta_k)^{\dagger}\nabla_\theta J_\rho(\theta_k)$. Inequalities $(b)$, and $(c)$ follow from Assumption \ref{ass_score}(a) and \ref{ass_4} respectively. Choosing $\eta = \frac{\mu_F^2}{4G^2L}$, taking a sum over $k=0,\cdots, K-1$, and rearranging the resulting terms, we get the following.
 \begin{align}
 \begin{split}
     \dfrac{\mu_F^2}{16G^4 L}\left(\dfrac{1}{K}\sum_{k=0}^{K-1}\Vert\nabla_\theta J_\rho(\theta_k)\Vert^2\right)&\leq \dfrac{J_\rho(\theta_K)-J_\rho(\theta_0)}{K} + \left(\dfrac{\mu_F^2}{8L}+\dfrac{\mu_F^4}{16G^4 L}\right)\left(\dfrac{1}{K}\sum_{k=0}^{K-1}\Vert\omega_k-\omega_k^*\Vert^2\right)\\
     &\overset{(a)}{\leq} \dfrac{2}{(1-\gamma)K}+\left(\dfrac{\mu_F^2}{8L}+\dfrac{\mu_F^4}{16G^4 L}\right)\left(\dfrac{1}{K}\sum_{k=0}^{K-1}\Vert\omega_k-\omega_k^*\Vert^2\right)
 \end{split}
 \end{align}
where $(a)$ uses the fact that $J_\rho(\cdot)$ is absolutely bounded above by $1/(1-\gamma)$. This proves the desired result.

\section{Proof of Lemma \ref{lemma_noise_variance}}
Note that,
\begin{align}
    \begin{split}
        \mathbf{E}&\left[\hat{\nabla}_{\omega}L_{\nu^{\pi_\theta}_\rho}(\omega_{\theta}^*, \theta)\otimes\hat{\nabla}_{\omega}L_{\nu^{\pi_\theta}_\rho}(\omega_{\theta}^*, \theta)\right] \\
        &= \mathbf{E}_{s\sim d_\rho^{\pi^*}}\mathbf{E}_{a\sim\pi^*(\cdot\vert s)}\bigg[\mathbf{E}\bigg[\underbrace{\nabla_\theta\log\pi_{\theta}(a\vert s)\cdot\omega^*_{\theta}-\dfrac{1}{1-\gamma}\hat{A}^{\pi_{\theta}}(s,a)}_{\triangleq \zeta_{\theta}(s, a)}\bigg]^2\nabla_\theta\log\pi_\theta(a|s)\otimes \nabla_\theta\log\pi_\theta(a|s)\bigg]
    \end{split}
\end{align}

Therefore, it is sufficient to show that $\mathbf{E}[\zeta_\theta(s, a)]\leq \sigma^2$, $\forall (s, a)\in \mathcal{S}\times \mathcal{A}$. Observe the following chain of inequalities,
\begin{align}
\label{eq:eq_41_appndx}
    \begin{split}
        \mathbf{E}&\bigg[\nabla_\theta\log\pi_{\theta}(a\vert s)\cdot\omega^*_{\theta}-\dfrac{1}{1-\gamma}\hat{A}^{\pi_{\theta}}(s,a)\bigg]^2\\
        &\leq 2\left[\nabla_\theta\log\pi_{\theta}(a\vert s)\cdot\omega^*_{\theta}\right]^2 + \dfrac{2}{(1-\gamma)^2} \mathbf{E}\left[\hat{A}^{\pi_\theta}(s, a)\right]^2\\
        &\overset{(a)}{\leq} 2\Vert\nabla_\theta\log\pi_\theta(a|s)\Vert^2 \Vert\omega_\theta^*\Vert^2 + \dfrac{4}{(1-\gamma)^2} \mathbf{E}\left[\hat{Q}^{\pi_\theta}(s, a)\right]^2 + \dfrac{4}{(1-\gamma)^2} \mathbf{E}\left[\hat{V}^{\pi_\theta}(s)\right]^2\\
        &\overset{(b)}{\leq} 2G^2\Vert F_\rho(\theta)^{\dagger}\nabla_\theta J_\rho(\theta)\Vert^2 + \dfrac{32}{(1-\gamma)^4}\\
        &\overset{(c)}{\leq} \dfrac{2G^4}{\mu_F^2(1-\gamma)^4} + \dfrac{32}{(1-\gamma)^4}
    \end{split}
\end{align}
Relation $(a)$ is a consequence of Cauchy-Schwarz inequality and the fact that $(a-b)^2\leq 2(a^2+b^2)$ for any two reals $a, b$. Inequality $(b)$ uses Assumption \ref{ass_score}, the definition that $\omega_\theta^*=F_\rho(\theta)^{\dagger}\nabla_\theta J_\rho(\theta)$, and the following two trivial bounds whereas $(c)$ follows from Assumption \ref{ass_4} and Lemma \ref{lemma_2}.
\begin{align}
\label{eq:eq_42}
    \mathbf{E}\left[\hat{Q}^{\pi_\theta}(s, a)\right]^2\leq \dfrac{4}{(1-\gamma)^2}, ~\text{and}~\mathbf{E}\left[\hat{V}^{\pi_\theta}(s)\right]^2\leq \dfrac{4}{(1-\gamma)^2}, ~~\forall(s, a)\in \mathcal{S}\times\mathcal{A}
\end{align}
To establish the first bound, note that $\hat{Q}^{\pi_{\theta}(s, a)}$ is assigned a zero value with probability $1/2$ and a positive value of at most $2(j+1)$ with probability $\frac{1}{2}(1-\gamma)\gamma^j$. Therefore,
\begin{align}
    \mathbf{E}\left[\hat{Q}^{\pi_\theta}(s, a)\right]^2\leq \dfrac{1}{2}\sum_{j=0}^\infty 4(1-\gamma)(j+1)^2 \gamma^j = \dfrac{2(1+\gamma)}{(1-\gamma)^2} < \dfrac{4}{(1-\gamma)^2}
\end{align}
The second bound in $(\ref{eq:eq_42})$ can be proven similarly. This concludes the lemma.

\section{Proofs of Lemma \ref{lemma_second_order} and \ref{lemma_first_order}}

We shall prove Lemma \ref{lemma_second_order} and \ref{lemma_first_order} using Corollary 2 of \citep{jain2018accelerating}. Note the following statements.

$\mathbf{S1}:$ The following quantities exist and are finite $\forall \theta\in\mathbb{R}^{\mathrm{d}}$.
\begin{align}
    &F_{\rho}(\theta) \triangleq \mathbf{E}_{(s, a)\sim \nu^{\pi_{\theta}}_\rho}\big[\nabla_\theta \log\pi_{\theta}(a|s)\otimes \nabla_\theta \log\pi_{\theta}(a|s)\big],\\
    &G_{\rho}(\theta) \triangleq \mathbf{E}_{(s, a)\sim \nu^{\pi_{\theta}}_\rho}\big[\nabla_\theta \log\pi_{\theta}(a|s)\otimes \nabla_\theta \log\pi_{\theta}(a|s)\otimes \nabla_\theta \log\pi_{\theta}(a|s)\otimes \nabla_\theta \log\pi_{\theta}(a|s)\big]
\end{align}
$\mathbf{S2}:$ There exists $\sigma^2>0$ such that the following is satisfied $\forall \theta \in \mathbb{R}^{\mathrm{d}}$ where $\omega_\theta^*$ is a maximizer of $L_{\nu^{\pi_\theta}_\rho}(\cdot, \theta)$.
\begin{align}
    \label{eq:eq_46_appndx}
    \begin{split}
        \mathbf{E}&\left[\hat{\nabla}_{\omega}L_{\nu^{\pi_\theta}_\rho}(\omega_{\theta}^*, \theta)\otimes\hat{\nabla}_{\omega}L_{\nu^{\pi_\theta}_\rho}(\omega_{\theta}^*, \theta)\right]\preccurlyeq \sigma^2 F_\rho(\theta)
    \end{split}
\end{align}
$\mathbf{S3:}$ There exists $\mu_F, G>0$ such that the following statements hold $\forall \theta \in \mathbb{R}^{\mathrm{d}}$.
\begin{align}
    (a)~&F_\rho(\theta)\succcurlyeq \mu_F I_{\mathrm{d}},\\
    (b)~& \mathbf{E}_{(s, a)\sim \nu^{\pi_\theta}_\rho}\left[\Vert \nabla_\theta \log\pi_{\theta}(a|s) \Vert^2 \nabla_\theta \log\pi_{\theta}(a|s)\otimes \nabla_\theta \log\pi_{\theta}(a|s)\right]\preccurlyeq G^2 F_\rho(\theta), \\
    (c)~& \mathbf{E}_{(s, a)\sim \nu^{\pi_\theta}_\rho}\left[\Vert \nabla_\theta \log\pi_{\theta}(a|s) \Vert^2_{F_\rho(\theta)^{\dagger}} \nabla_\theta \log\pi_{\theta}(a|s)\otimes \nabla_\theta \log\pi_{\theta}(a|s)\right]\preccurlyeq \dfrac{G^2}{\mu_F} F_\rho(\theta)
\end{align}
Statement $\mathbf{S1}$ follows from Assumption \ref{ass_score}(a) whereas $\mathbf{S2}$ is established in Lemma \ref{lemma_noise_variance}. Statement $\mathbf{S3}(a)$ is identical to Assumption \ref{ass_4}, $\mathbf{S3}(b)$ follows from Assumption \ref{ass_score}(a), and $\mathbf{S3}(c)$ is a result of Assumption \ref{ass_score}(a) and \ref{ass_4}. Therefore, we can apply Corollary 2 of \citep{jain2018accelerating} with $\kappa=\Tilde{\kappa}=G^2/\mu_F$ and obtain the convergence result stated below with $H>\bar{C}\sqrt{\kappa\Tilde{\kappa}}\log(\sqrt{\mathrm{d}}\sqrt{\kappa\Tilde{\kappa}})$ iterations of the inner loop, and the following learning rates $\alpha = \frac{3\sqrt{5} \sqrt{\kappa\Tilde{\kappa}}}{1+3\sqrt{5\kappa\Tilde{\kappa}}}$, $\beta = \frac{1}{9\sqrt{\kappa\Tilde{\kappa}}}$, $\xi=\frac{1}{3\sqrt{5}\mu_F\sqrt{\kappa\Tilde{\kappa}}}$, and $\delta = \frac{1}{5G^2}$. 
\begin{align}
\label{eq:eq_50_appndx}
\begin{split}
    & \mathbf{E}\left[l_k(\omega_k)\right] - l_k(\omega_k^*) \leq \dfrac{C}{2}\exp\left(-\dfrac{H}{20\sqrt{\kappa\Tilde{\kappa}}}\right)\left[l_k(\mathbf{0})-l_k(\omega_k^*)\right]+11\dfrac{\sigma^2\mathrm{d}}{H}, \\
    &\text{where~} l_k(\omega) \triangleq L_{\nu^{\pi_\theta}_\rho}(\omega, \theta)|_{\theta=\theta_k}, ~\forall \omega\in\mathbb{R}^{\mathrm{d}}
\end{split}
\end{align}
The term, $C$ is a universal constant. Note that $l_k(\omega_k^*)\geq 0$ and $l_k(\mathbf{0})$ is bounded above as follows. 
\begin{align}
\label{eq:eq_51_appndx}
    \begin{split}
        l_k(\mathbf{0})=\dfrac{1}{2}\mathbf{E}_{(s, a)\sim \nu^{\pi_{\theta_k}}_\rho}&\bigg[\dfrac{1}{1-\gamma}A^{\pi_{\theta_k}}(s,a)\bigg]^2\overset{(a)}{\leq} \dfrac{1}{2(1-\gamma)^4}
    \end{split}
\end{align}
where $(a)$ follows from the fact that $|A^{\pi_{\theta}}(s, a)|\leq 1/(1-\gamma)$, $\forall (s, a)$, $\forall \theta$. Combining $\eqref{eq:eq_50_appndx}$,  $\eqref{eq:eq_51_appndx}$, and the fact that $l_k(\cdot)$ is $\mu_F$-strongly convex, we establish,
    \begin{align}
    \label{eq:eq_appndx_52}
    \begin{split}
        \mathbf{E}\Vert\omega_k - \omega_k^*\Vert^2 \leq \frac{2}{\mu_F} \big[\mathbf{E}\left[l_k(\omega_k)\right] - l_k(\omega_k^*)\big]&\leq 22\dfrac{\sigma^2\mathrm{d}}{\mu_F H}
        + C\exp\left(-\dfrac{\mu_F }{20G^2}H\right)\left[\dfrac{1}{\mu_F(1-\gamma)^4}\right]
    \end{split}
    \end{align}
This concludes Lemma \ref{lemma_second_order}. For noiseless ($\sigma^2=0$) gradient updates, we get the following.
    \begin{align}
    \label{eq:eq_appndx_53}
    \begin{split}
        \mathbf{E}\Vert(\mathbf{E}[\omega_k|\theta_k] - \omega_k^*)\Vert^2 &\leq  C\exp\left(-\dfrac{\mu_F }{20G^2}H\right)\left[\dfrac{1}{\mu_F(1-\gamma)^4}\right]
    \end{split}
    \end{align}
Lemma \ref{lemma_first_order}, therefore, can be established from $\eqref{eq:eq_appndx_53}$ by applying Jensen's inequality on the function $f(x)=x^2$.

\section{Proof of Theorem \ref{theorem_final}}

Combining Lemma \ref{lemma_second_order} and \ref{lemma_first_order} with Corollary \ref{corr_1}, we get the following bound for appropriate choices of the learning rates as mentioned in the theorem.
\begin{equation}\label{eq:eq_appndx_54}
		\begin{split}
			J_{\rho}^{*}-\frac{1}{K}\sum_{k=0}^{K-1}\mathbf{E}[J_{\rho}(\theta_k)]
            &\leq \sqrt{\epsilon_{\mathrm{bias}}}
            + G\left\lbrace \sqrt{C}\exp\left(-\dfrac{\mu_F }{40G^2}H\right)\left[\dfrac{1}{\sqrt{\mu_F}(1-\gamma)^2}\right]\right\rbrace\\
			&+\dfrac{B}{4L}\left(\dfrac{\mu_F^2}{G^2}+G^2\right)\left\lbrace C\exp\left(-\dfrac{\mu_F }{20G^2}H\right)\left[\dfrac{1}{\mu_F(1-\gamma)^4}\right]+22\dfrac{\sigma^2\mathrm{d}}{\mu_F H}\right\rbrace\\
            &+\dfrac{G^2}{\mu_F^2K}\left(\dfrac{B}{1-\gamma}+4L\mathbf{E}_{s\sim d_\rho^{\pi^*}}[KL(\pi^*(\cdot\vert s)\Vert\pi_{\theta_0}(\cdot\vert s))]\right)
        \end{split}
	\end{equation}

Using the fact that $\exp(-x)<\frac{1}{x}$, and $\exp(-x)<1$, $\forall x>0$, we get the following result.
\begin{align}
\label{eq:eq_appndx_55}
    \begin{split}
        J_{\rho}^{*}-\frac{1}{K}\sum_{k=0}^{K-1}\mathbf{E}[J_{\rho}(\theta_k)]
            \leq \sqrt{\epsilon_{\mathrm{bias}}}
            &+ \underbrace{\left[\dfrac{40 G^3\sqrt{C}}{\mu_F^{3/2}(1-\gamma)^2}+\dfrac{B}{4L}\left(\dfrac{\mu_F^2}{G^2}+G^2\right)\left\lbrace \dfrac{20G^2C}{\mu_F^2(1-\gamma)^4}+22\dfrac{\sigma^2\mathrm{d}}{\mu_F }\right\rbrace\right]}_{\triangleq P} \dfrac{1}{H}\\
            &+\underbrace{\left\lbrace \dfrac{G^2}{\mu_F^2}\left(\dfrac{B}{1-\gamma}+4L\mathbf{E}_{s\sim d_\rho^{\pi^*}}[KL(\pi^*(\cdot\vert s)\Vert\pi_{\theta_0}(\cdot\vert s))]\right)\right\rbrace }_{\triangleq Q}\dfrac{1}{K}
    \end{split}
\end{align}
To get an optimality gap of $\sqrt{\epsilon_{\mathrm{bias}}}$, we must choose $H=\frac{2P}{\epsilon}$ and $K=\frac{2Q}{\epsilon}$. Note that,  $H>\bar{C}\frac{G^2}{\mu_F}\log\left(\sqrt{d}\frac{G^2}{\mu_F}\right)\triangleq H_0$ can be achieved if $\epsilon<\frac{2P}{H_0}$. In other words, $\epsilon$ must be sufficiently small for the theorem to be true. To determine the sample and iteration complexities as functions of $\gamma$, note the following observations.
\begin{itemize}
    \item The terms $B, G, C, \mu_F, \mathrm{d}$, and $\mathbf{E}_{s\sim d_\rho^{\pi^*}} KL(\pi^*(\cdot\vert s)\Vert\pi_{\theta_0}(\cdot\vert s))$ are constants i.e., independent of $\gamma$.
    \item Lemma \ref{lemma_noise_variance} demonstrates that $\sigma^2 = \mathcal{O}((1-\gamma)^{-4})$.
    \item According to Lemma \ref{lemma_2}, $L=\mathcal{O}((1-\gamma)^{-3})$ and $1/L=\mathcal{O}((1-\gamma)^3)$.
\end{itemize}

Thus, $H=\mathcal{O}((1-\gamma)^{-2}\epsilon^{-1})$ and $K=\mathcal{O}((1-\gamma)^{-3}\epsilon^{-1})$. Note that the expected number of steps executed by Algorithm \ref{algo_sampling} is $\mathcal{O}((1-\gamma)^{-1})$. Therefore, the sample complexity is $\mathcal{O}((1-\gamma)^{-1}HK)=\mathcal{O}((1-\gamma)^{-6}\epsilon^{-2})$. Moreover, the iteration complexity is $\mathcal{O}(K)=\mathcal{O}((1-\gamma)^{-3}\epsilon^{-1})$.

\section{Auxiliary Lemma}

 \begin{lemma} 
 \label{lemma_aux_1}
 \citep[Lemma 2]{agarwal2021theory} The following relation holds for any two policies $\pi_1$, $\pi_2$.
 \begin{align} 
     J_{\rho}^{\pi_1} - J_{\rho}^{\pi_2} = \dfrac{1}{1-\gamma} \mathbf{E}_{(s, a)\sim \nu^{\pi_1}_{\rho}}\left[A^{\pi_2}(s, a)\right] 
 \end{align}
 \end{lemma}
%\fi
\end{document}